\pdfoutput=1

\documentclass[11pt,a4paper]{article}
\usepackage{times, latexsym}
\usepackage{xurl}
\usepackage[T1]{fontenc}
\usepackage[utf8]{inputenc}
\usepackage{CJKutf8}
\usepackage{booktabs, multirow, diagbox}
\usepackage{graphicx, svg, subfigure}
\usepackage{tipa}
\usepackage{amsmath,amsthm,textcomp,amssymb,geometry,enumerate,makecell,extarrows, pifont}
\usepackage{algorithm}
\usepackage{algpseudocode}
\usepackage{threeparttable, array}
\usepackage[numbers]{natbib}
\usepackage{todonotes}
\usepackage{hyperref}
\usepackage{breakurl}

\hypersetup{
    colorlinks=true,
    linkcolor=blue,
    filecolor=blue,      
    urlcolor=blue,
    citecolor=cyan,
}
\usepackage{tikz}
\usepackage{tikz-qtree}
\usetikzlibrary{decorations.pathreplacing}
\usepackage{comment}
\usepackage{scalefnt}
\usetikzlibrary {arrows.meta}
\usetikzlibrary{decorations.pathreplacing,calligraphy}

\setcitestyle{authoryear,round}

   \usepackage[acceptedWithA]{tacl2021v1}
\usepackage[acceptedWithA]{tacl2021v1}
\usepackage{xspace,mfirstuc,tabulary}

\newif\iftaclinstructions
\taclinstructionsfalse 
\iftaclinstructions
\renewcommand{\confidential}{}
\renewcommand{\anonsubtext}{(No author info supplied here, for consistency with TACL-submission anonymization requirements)}
\newcommand{\instr}
\fi

\iftaclpubformat 

\else

\fi

\newcommand{\ch}[1]{\begin{CJK}{UTF8}{bsmi} #1 \end{CJK}}
\newcommand{\setlen}{\setlength{\abovedisplayskip}{3pt} 
\renewcommand{\TPTminimum}{1.03\linewidth}
\setlength{\belowdisplayskip}{3pt}}
\allowdisplaybreaks[4] 

\newtheorem{theorem}{Proposition}

\title{Phonetic Reconstruction of the Consonant System of Middle Chinese via Mixed Integer Optimization}

\author{
  Xiaoxi Luo 
  \\
  Yuanpei College, Peking University
  \\
  \texttt{lxx\_1900017744@pku.edu.cn}
  \And
  Weiwei Sun
  \\
  Department of Computer Science and\\ Technology,
  Cambridge University
  \\
  \texttt{ws390@cam.ac.uk}
}

\newcommand{\qy}{Qi\=ey\`un}
\newcommand{\fq}{F\v anqi\=e}
\newcommand{\gy}{Gu\v angy\`un}
\newcommand{\zhiyin}{Zh\'iy\textipa{\=\i}n}

\begin{document}
\begin{CJK}{UTF8}{bsmi}

\maketitle 

\begin{abstract}

This paper is concerned with phonetic reconstruction of the consonant system of Middle Chinese.
We propose to cast the problem as a Mixed Integer Programming problem, which is able to automatically explore homophonic information from ancient rhyme dictionaries and phonetic information from modern Chinese dialects, the descendants of Middle Chinese. 
Numerical evaluation on a wide range of synthetic and real data demonstrates the effectiveness and robustness of the new method.
We apply the method to information from \gy~and 20 modern Chinese dialects to obtain a new phonetic reconstruction result.
A linguistically-motivated discussion of this result is also provided.\footnote{Code and datasets are available at \url{https://github.com/LuoXiaoxi-cxq/Reconstruction-of-Middle-Chinese-via-Mixed-Integer-Optimization}.}

\end{abstract}

\section{Introduction}
Phonological reconstruction is one main concern in historical linguistics. 
There are two fundamental goals: reconstructing phonological categories and reconstructing phonetic values of these categories or individual phonemes. 
The classic linguistic and philological approach applies a comparative strategy to solve both problems by connecting cognates in different languages.
Previous research in computational linguistics demonstrates the possibility to automate the comparative approach to some extent.
See e.g. \citeauthor{bouchard-2007a} (\citeyear{bouchard-2007a}, \citeyear{bouchard-2009}, \citeyear{bouchard-2013}), \citet{list-etal-2022-new}, and \citet{he-etal-2023-neural}, among others.

The comparative approach developed out of attempts to reconstruct Proto-Indo-European, the common ancestor of the Indo-European language family.
However, the comparative method itself is not well equipped to handle the special challenges of reconstructing Chinese. 
On the one hand, a key step in the comparative method is to identify as many cognates as possible, which is relatively straightforward for Chinese but can be extremely challenging in other languages.
On the other hand, documentary materials predominantly use Chinese characters to annotate other characters (such as \fq~反切), a tradition that has continued for thousands of years. 
For example, rhyme dictionaries such as \qy~切韻 extensively use this unique annotation method and systematically represent the phonological system of Chinese during a specific period. 
Such precious materials are relatively rare in other languages. 
This information is invaluable for the reconstruction of proto-languages, but the comparative method itself cannot adequately handle it. 
In fact, throughout Chinese history, numerous works similar to the \qy~have existed in different periods, each reflecting the phonological system of its time.
The purpose of this paper is to address the question of how to systematically utilize these phonetic materials.
    
In the practice of phonological reconstruction for ancient Chinese, linguists have been overwhelmingly exploring alternative information to spelling and hence alternative methods to the comparative one \citep[pp.1--2]{huang2014handbook}.
Their work heavily relies on philological documents, especially rhyme dictionaries, which have a unique way to record homophonic information, i.e. \fq. 
A basic consensus on the phonological categories of Middle Chinese (MC)\footnote{There are three basic periods: Old Chinese, Middle Chinese and Old Mandarin \citep{wangli-1957}.} has been reached---there were 35--38 initials in MC, with minor disagreement only on some categories' merging or splitting \citep{gbh,lr-1956,wangli-1957}.
Ancient rhyme dictionaries, however, do not provide phonetic information, and phonetic reconstruction is still extremely challenging.
The relevant research is limited and there is a lot of disagreement among scholars. 
Figure \ref{fig:confu-matrix} shows the percentage of initials with which scholars disagree. 
There is significant inconsistency between any two scholars, let alone a consensus among all of them.

\begin{figure}[t]
    \centering
    \includegraphics[width=0.99\linewidth]{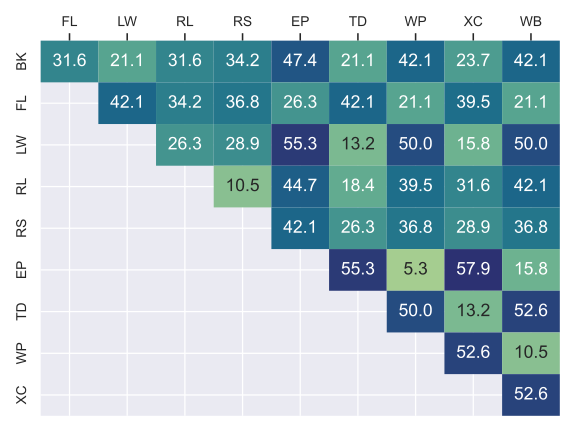}
    \caption{Disagreement among scholars on phonetic reconstruction. BK: \citet{gbh}, FL: \citet{lfk-1971}, 
    LW: \citet{wangli-1957}, RL: \citet{lr-1956}, RS: \citet{shaorongfen}, EP: \citet{pulleyblanks}, TD: \citet{dth-2004}, WP: \citet{panwuyun-2000}, XC: \citet{chenxinxiong}, WB: \citet{Baxter1992}.\footnote{The original data comes from \url{https://zh.wikipedia.org/wiki/\%E4\%B8\%AD\%E5\%8F\%A4\%E9\%9F\%B3}.} The number in each cell represents the percentage of initials on which the corresponding two scholars disagree. 
}
    \label{fig:confu-matrix}
\end{figure}

This paper is concerned with developing a computational model for phonetic reconstruction of the consonant system of MC.
We propose to cast phonetic reconstruction as a Mixed Integer Optimisation problem (\S\ref{sec:model}). 
A particular goal is to conveniently integrate heterogeneous information. 
We consider two major information sources: 
(1) philological documents and 
(2) modern Chinese dialects\footnote{The modern Chinese dialects are more like a family of languages \citep{huang2014handbook}, and many of them are not mutually intelligible. 
This paper uses the term `dialect' instead of `variety', because we focus on their common ancestor MC.}, the descendants of MC.
Following Generative Phonology \cite{chomsky-1968},
we introduce a novel compact set of Chinese-specific distinctive features to represent consonants (\S\ref{sec:feature-set}).
Based on the feature-oriented precise phonetic representation, we formalise 
the optimisation goal as minimising the overall distance between possible homophonic characters and the overall distance between MC and modern dialects.
Measuring the distance is a key element to the success of the new architecture.
To this end, we design a new mathematically sound distance/metric function to suit our feature representation.

Evaluating the {goodness} of reconstruction result is uniquely challenging because of the lack of ground-truth.
Instead, we evaluate the reconstruction method.
We consider two types of experiments: experiments on synthetic data (\S\ref{sec:synthesis}), where the ground-truth is known, and experiments with held-out data (\S\ref{sec:experiment}), where partial information transformed from the ground-truth is known.
To create representative synthetic data, we start from a pre-defined consonant system, derive homophonic information that matches \fq, and derive varieties by introducing stochastic change as well as random noise. 
We consider three types of consonant systems: 1) purely artificial systems that randomly select elements, e.g. from an IPA chart, 
2) natural systems of modern languages, including English, German and Mandarin, and 3) the reconstructed system of Latin.
Numerical evaluation demonstrates the effectiveness and robustness of the new method.
It is able to successfully reconstruct most consonants when natural and reconstructed consonant systems are considered. 
 
For the experiments with real data, we consider a wide range of representative Chinese characters with relevant information from \gy~and 20 modern dialects.
Given the absence of ground truth in phonetic reconstruction, 
to validate the effectiveness of the reconstruction method, we employ the strategy to hold out some \fq~information.
In particular, we apply our method to 70\% \fq~annotations and compare the automatically reconstructed result with the other 30\%.
The reconstructed phonemes predict around 68\% \fq.
Considering that \fq~annotations themselves are not fully consistent,
the result is quite promising and the method has a potential use to detect inconsistent \fq~annotations.

Based on the entire real data set, we provide a new phonetic reconstruction for Middle Chinese.
We present both numerical and linguistic comparison to previous philologist work (\S\ref{sec:main-result}).
Our phonetic reconstruction aligns to the well-studied phonological category reconstruction to a great extent---it obtains an Adjusted Mutual Information \citep{AMI-2010} score of over 0.8. 
A linguistic analysis of the reconstruction result suggests some future research venues. 

\section{Linguistic and Philological Basis} \label{ling-philo-basis}

\subsection{Syllable Structure} \label{sec:syllable-structure}

Ancient documents overwhelmingly indicate that Chinese was, from the beginning of its recorded history, a monosyllabic language,
 in which morphemes are by and large represented by single syllables \citep{norman1988,shen_2020}.
Moreover, the sound pattern of the syllabic structure remained unchanged from Middle Chinese to modern Mandarin.
The syllabic structure is composed of an initial segmental consonant (I), 
a medial (aka on-glide, denoted as M hereafter), a main vowel (V), a coda (or an off-glide), denoted by C hereafter, and a suprasegmental tone (T).
The terms `rime' and `final' are also frequently used: 
Rime is the combination of the main vowel and the coda, while final is a combination of the medial, the main vowel and the coda.
There are no consonant clusters, i.e. more than one consecutive consonants.
See Figure \ref{fig:syllabic-structure} for the hierarchical organisation of the above elements. 
Below we list three examples:
\begin{itemize} \setlength\itemsep{0.01em}
  \item \ch{巔}/\text{[tian]}: I=t, M=i, V=a, C=n, T=55 
  \item \ch{眼}/\text{[ian]}: I=$\emptyset$, M=i, V=a, C=n, T=214
  \item \ch{暗}/\text{[an]}: I=$\emptyset$, M=$\emptyset$, V=a, C=n, T=51
\end{itemize}


\begin{figure}[hbtp]
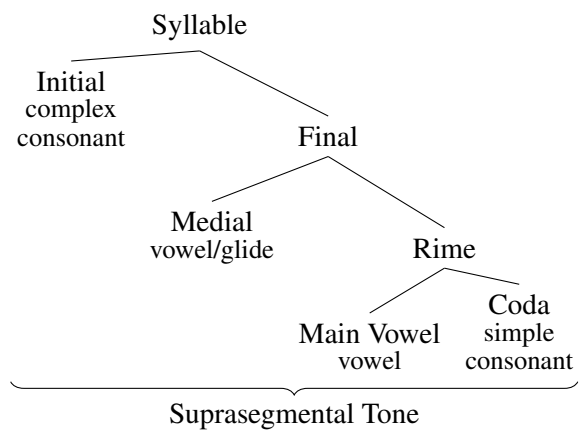

  \centering
  \tikz{
    \tikzset{level distance=42pt}
    \Tree
      [.Syllable
        \node(l){\shortstack{Initial\\\small{complex}\\\small{consonant}}};
        [.Final
          {\shortstack{Medial\\\small{vowel/glide}}}
          [.Rime
            {\shortstack{Main Vowel\\\small{vowel}}}
            \node(r){\shortstack{Coda\\\small{simple}\\\small{consonant}}};
          ] ] ] 
    \draw [decorate,decoration={brace,amplitude=5pt,mirror,raise=4ex}] (-2.5,-3.9) -- (5,-3.9) node[midway,yshift=-3em]{Suprasegmental Tone};
  }

  \caption{The syllabic structure of MC and Mandarin.}
  \label{fig:syllabic-structure}
\end{figure}

Consonants can only appear as I or C. 
Consonantal codas are rather simple and have been relatively clearly recorded in rhyme dictionaries. 
The reconstruction of the associated phonetic values is also clear: 6 categories in total, including nasals [m, n, \textipa{\ng}] and stops [p, t, k]). 
This paper aims to complete the reconstruction of the entire consonant system by systematically studying initials.

\subsection{\fq~Spelling}
\fq~is a traditional method to indicate the pronunciation of a character in question. 
In the \fq~spelling, two characters are selected as two spellers to represent the pronunciation of the character (denoted as $X$) in question: 
the first character ($X_u$) is called the upper speller and shares the same I with $X$; 
the second character ($X_l$) is called the lower speller and shares the same M, V, E, and T with $X$. 
Take Figure \ref{fig:exp-fanqie} for example. 
To partially record the pronunciation of 烘, 戶 is employed as the upper speller, while 公 is used as the lower speller.
\begin{figure}[hbtp]
    \centering
    \includegraphics[width=0.6\linewidth]{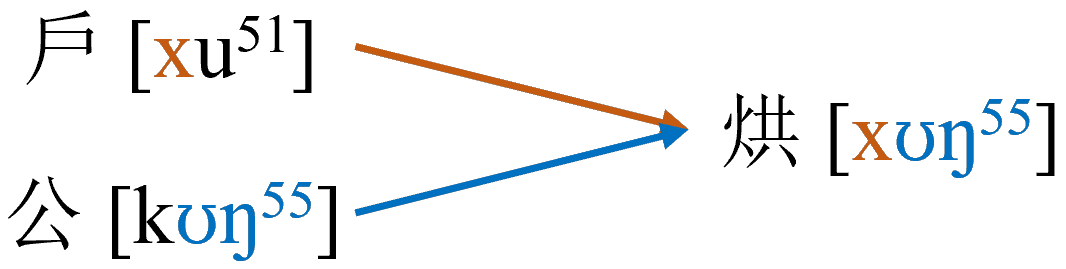}
    \caption{An example of the \fq~spelling. }
    \label{fig:exp-fanqie}
\end{figure}

\zhiyin~直音 is another method to partially annotate pronunciation. 
It uses a homophonic character to annotate the character in question. 
Both \fq~and \zhiyin~were frequently used. 

A rhyme dictionary is a type of ancient Chinese dictionary that collates characters by tone and rime. 
In rhyme dictionaries, there are three types of important phonological information: rhyme categories, \fq~spellings, and \zhiyin~notations. 
The \qy~is a renowned rhyme dictionary that encapsulates the phonology of MC. 
Chinese philologists have been working on it to derive phonological analysis for centuries.

\subsection{Modern Chinese Dialects} 
\label{intro-dialect}
Modern Chinese dialects are classified into seven groups in three geographic zones: 
Mandarin (northern zone), Wu, Min, Xiang (central zone), Gan, Hakka and Yue (southern zone) \citep{norman1988,dialectology}.
Most scholars believe that they are all descendants of MC, 
and therefore provide valuable information to reconstruct phonetic values. 

\begin{table*}[!t]
  \centering
    \scalebox{0.95}{
    \begin{tabular}{@{}c@{}cl}
    \toprule
    \multicolumn{1}{r}{} & \textbf{Name} & \multicolumn{1}{c}{\textbf{Value}} \\
    \midrule
    \multicolumn{1}{c}{\multirow{4}[6]{*}{\shortstack{Manner \\ Feature}}} & sonority & $5$: vowel, $4$: glide, $3$: liquid, $2$: nasal, $1$: obstruent \\
    \cmidrule{2-3}      & \multirow{2}[2]{*}{continuant} & $1$: fricatives, liquids, glides, laterals \\
          &       & $-1$: stops, affricates, nasals \\
    \cmidrule{2-3}      & {delayed release$^{\text{1}}$\tnote{1}} & $1$: fricatives, affricates, $-1$: stops \\
    \midrule
    \multicolumn{1}{c}{\multirow{11}[18]{*}{\shortstack{Place \\ Feature}}} & labial &  $1$: articulated with the lips \\
    \cmidrule{2-3}      & labiodental$^{\text{2}}$\tnote{2} & $1$: articulated by touching the lower lip to the upper teeth \\
    \cmidrule{2-3}      & coronal & $1$: articulated with the tongue blade and/or tip \\
    \cmidrule{2-3}      & anterior$^{\text{3}}$\tnote{3} & $1$: (front) dental, alveolar, $-1$: (post) palato-alveolar, retroflex\\
    \cmidrule{2-3}      & \multirow{2}[2]{*}{distributed$^{\text{3}}$\tnote{3}} & 1: (blade, laminal) dental, palato-alveolar \\
          &       & $-1$: (tip, apical) alveolar, retroflex \\
    \cmidrule{2-3}      & \multirow{2}[2]{*}{lateral} & $1$: distinguishes [l] from other coronal liquids and [\textipa{\textbeltl}, l\textipa{\textyogh}]  \\
          &       & from other coronal fricatives. \\
    \cmidrule{2-3}      & dorsal & $1$: articulated with the tongue body \\
    \cmidrule{2-3}      & high$^{\text{4}}$\tnote{4} & $3$: velar, $2$: uvular, $1$: pharyngeal \\
    \cmidrule{2-3}      & front$^{\text{4}}$\tnote{4} & $3$: fronted velar, $2$: central velar, $1$: back velar, uvular, pharyngeal \\
    \midrule
    \multicolumn{1}{c}{\multirow{2}[4]{*}{\shortstack{Laryngeal \\ Feature}}} & voice & $1$: voiced, $-1$: voiceless \\
    \cmidrule{2-3}      & spread glottis & $1$: [h], breathy vowels, and aspirated consonants. \\
    \bottomrule
  \end{tabular}
  }
  \caption{\label{table:feature-set}Our feature set.
    In the `Value' column, the number before the colon is the possible value of the feature, while the right of the colon is the condition for taking this value. 
    If there is only value `$1$', it means that the feature is `$-1$' under all other circumstances. 
    $^1$Only meaningful for obstruents (i.e., when sonority is $1$).
    $^2$Only meaningful when [+labial].
    $^3$Only meaningful when [+coronal].
    $^4$Only meaningful when [+dorsal].
    } 
\end{table*}

\section{Representing Phonemes} \label{sec:feature-set}
Representing phonemes in a formal way plays an essential role in computational reconstruction.
Following Generative Phonology \cite{chomsky-1968}, we use distinctive features to represent phonemes. 
\citet{hayes2011} proposes a feature set for all human languages.
Any specific language only uses a subset of it to mark phonemic contrasts. 
To compactly represent Chinese and its modern varieties, we propose a Chinese-specific set. 
Reducing the total number of distinctive features can also boost the efficiency in solving the corresponding optimisation problem. 

\subsection{Distinctive Features} \label{intro-gp}

Distinctive features provide a systematic way to identify and represent phonemes. 
Each phoneme is represented as and collectively defined by a bundle of binary features \citep[p.71]{hayes2011}. 
The negative ($-$) and the positive ($+$) annotations are used to indicate the absence or presence of a feature. 
Below is an example:
\begin{equation*} \label{pom}
\small
\text {Pom:}  \! =  \! \left[ \! \begin{array}{l}
- \text { syllabic} \\
- \text { sonorant} \\
+ \text { stop} \\
- \text { nasal} \\
+ \text { labial} \\
- \text { voice}
\end{array}  \! \! \right]  \! 
\left[  \! \begin{array}{l}
+ \text { syllabic} \\
+ \text { sonorant} \\
- \text { stop} \\
- \text { nasal} \\
+ \text { low} \\
+ \text { back} \\
- \text { round}
\end{array}  \! \! \right]  \! 
\left[ \! \begin{array}{l}
- \text { syllabic} \\
+ \text { sonorant} \\
+ \text { stop} \\
+ \text { nasal} \\
+ \text { labial} \\
+ \text { voice}
\end{array} \! \! \right] \end{equation*}

It is straightforward to formalise the bundle of features as a vector, which can be used to measure the distance between phonemes.

\subsection{Our Feature Set} \label{our-feature-set}

We propose the following modification of \citet{hayes2011} to obtain a feature set for Chinese.

\paragraph{Remove some features} 
Two types of features are not considered: (1) features that can be represented by other features, and (2) `tap' and `trill'\footnote{Taps, flaps and trills are uncommon in modern Chinese dialects \citep{zhuxiaonong-liquid}, and they do not appear in our dataset. 
No scholars have used taps, flaps and trills to reconstruct MC.
Existing research shows their close connection with the affix \'er 儿~\citep{trill-in-HB-2019}, but there is still no consensus on the timing and process of their formation.}. 

\paragraph{Merge some features}
The features in Generative Phonology are binary. 
By combining comparable and orderable features into multi-valued ones, we can reduce the number of features. 
For example, \citet{hayes2011} use 4 features (syllabic, consonantal, approximant, sonorant) to describe the sonority hierarchy, while we combine them into one feature, `sonority', with 5 graduable values. 

Our feature set is summarised in Table \ref{table:feature-set}. 
We have 14 features in total, reduced from 25 in \citet{hayes2011}. 
Accordingly, we use a 14-dimensional vector to represent a phoneme for computation. 

\paragraph{Independent vs dependent features}
In both the \citeauthor{hayes2011}' feature set and ours, some features are meaningful\footnote{The words `meaningful' and `meaningless' used here correspond to the description `not to care' in \citet[p.91]{hayes2011}: `in most languages with plain /p/, the position of the tongue body during the production of this sound is simply whatever is most articulatorily convenient, given the neighboring sounds. $\cdots$ the tongue body does not adopt any particular position during the /p/; $\cdots$ In this sense, the /p/ could be said truly `not to care' about values for dorsal features.'}
only when some other features at higher levels take certain values. 
We refer to features that decide whether other features are meaningful as `I-features' (independent feature), 
  and those determined by I-features as `D-features' (dependent feature).  

\paragraph{Zero value} \citet{hayes2011} uses the digit 0 (zero feature) to represent meaningless D-features. 
Some syllables in Chinese lack initial consonants, and \citet[pp.18--23]{chao1968} suggests to call them `zero initials'. 
Accordingly, we use digit 0 to represent zero initials.
0-valued I-features occur when and only when the corresponding character is initialless, 
  while 0-valued D-variables occur when and only when they are meaningless.

\newcommand{\fmc}{$F_{\text{MC}}(X)$}
\newcommand{\fmcu}{$F_{\text{MC}}(X_u)$}

\begin{table*}[th] 
  \centering
  \begin{threeparttable}
    \scalebox{0.95}{
    \begin{tabular}{|ll|}
    \hline
    {$X$} & character of which the initial is to be reconstructed\\
    {$X_u$} & upper speller of character $X$, with its initial to be reconstructed \\
    {$S_{fq}$} & set of all character--speller pairs $(X, X_u)$ \\
    {$L$/$l$} & set of modern dialects/a modern dialect \\
    {$F$} & 14-dimensional phonetic feature vector \\
    {$F^j$} & $j$-th dimension of phonetic feature vector $F$ \\
    {$F_l(X)$} & phonetic feature vector that encodes $X$'s initial in dialect $l$ (known) \\
    \fmc &  phonetic feature vector that encodes $X$'s initial in MC (to be solved) \\
    {$S_I$/$S_D$} & set of independent/dependent features \\
    {$\tau$} & function that maps D-feature $j$ to the corresponding I-feature $\tau(j)$\\
    {$d(F_1,F_2)$} & distance between feature vectors $F_1$ and $F_2$\\
    {$f$} & general distance function, e.g. $p$-norm\\
    {$g_{j,\tau(j)}(F_1, F_2)$} & distance function between $F_1$ and $F_2$ according to D-feature $j$ and I-feature $\tau(j)$\\\hline
    \end{tabular}%
    }
    \caption{\label{table:notation}A summary of mathematical notations used to illustrate our model.}
    \end{threeparttable}
\end{table*}

\section{The Optimisation Model} \label{sec:model}

Mixed Integer Programming (MIP) is an optimization problem in which some but not necessarily all variables are constrained to be integers.
MIP has been widely applied in many NLP tasks, e.g. dependency parsing \citep{riedel-clarke-2006-ILP-parsing}, semantic role labeling \citep{ILP-semantic-role}, coreference resolution \citep{de-belder-moens-2012-coreference}, as well as some more recent applications, e.g. exemplar selection for in-context learning. \citep{tonglet-etal-2023-seer}

We introduce our MIP model for phonetic reconstruction as follows. 
The objective function and constraints are detailed in \S\ref{sec:obj} and \S\ref{sec:restriction} separately. 
An essential component of the objective function is measuring the distance between two phonetic feature vectors, for which we propose a mathematically-sound distance function in \S\ref{sec:dis-f}.
Mathematical notations used in \S\ref{sec:obj}--\S\ref{sec:restriction} are summarised in Table \ref{table:notation}.

\subsection{The Objective Function} \label{sec:obj}

To phonetically reconstruct MC, we consider two information sources: \fq/\zhiyin~and varieties. 
\fq/\zhiyin~reveals homophonic relationships between characters of MC, 
and each descendent dialect partially reflects MC's phonetic structure. 
Formally, assume we have a set of characters under consideration, denoted as $S$. 
The construction of $S$ is discussed in \S\ref{sec:dataset}.
Each character $X\in S$ has at least one upper \fq~or \zhiyin~speller, denoted as $X_{u}\in S$. 
We collect all character--speller pairs and define the set $S_{\text{fq}}=\{(X, X_{u}): X \in S\}$.
Let $L$ denote the set of modern dialects.
The pronunciation of any character $X\in S$ in any dialect $l\in L$ is known.
Accordingly, the phonetic feature vector of $X$'s initial in $l$, denoted as $F_l(X)$, is known.
The goal is to infer its phonetic feature vector of MC, denoted as \fmc, based on $S_{\text{fq}}$ and all known $F_{l}(X)$ where $l\in L$.


We cast the goal as \textbf{minimising} the overall \textit{distance} between \fmc~and \fmcu,
and minimising the overall \textit{distance} between \fmc~and $F_l(X)$, for all $X\in S$.
Assume $d$ is a mathematically sound distance/metric function and $\lambda_{\text{fq}}\in(0,1)$ is a coefficient then the objective is
\begin{equation}\label{eq:objective}
  \begin{aligned}
    & \lambda_{\text{fq}} \sum_{(X,X_u) \in S_{\text{fq}}} 
       d(F_{\text{MC}}(X), F_{\text{MC}}(X_u))\\
    & \quad+ (1-\lambda_{\text{fq}}) \sum_{l\in L, X \in S} d(F_{\text{MC}}(X), F_l(X))
  \end{aligned}
\end{equation}
Although the speller $X_{u}$ is supposed to share the same initial with $X$ in general, 
we should not model such homophonic relation with constraint $F_{\text{MC}}(X)=F_{\text{MC}}(X_u)$ due to the existence of a considerable number of counterexamples.
Such inconsistency exists probably because the \fq/\zhiyin~spellings were not devised by one individual but rather collected from various preexisting phonological works, and therefore encoded phonological information of a mixture of diachronically connected languages \citep{shen_2020}.
Instead, we relax the identity restriction by employing a more general distance notion.


\subsection{The Distance Function} \label{sec:dis-f}

For each $X \in S$, we set 14 \textbf{continuous} variables $F^j (0\leqslant j\leqslant 13)$ to encode the phonetic value of its initial, 
each dimension corresponding to a feature. 
The range of $F^j$ is $[\min\{0,l_j\}, u_j]$, where $l_j$ and $u_j$ are the upper and lower bounds of its corresponding feature in Table \ref{table:feature-set}.
Usually, we can use $p$-norm to measure the distance between two real vectors.
However, in our problem, some features are not independent from each others --- it is meaningless to discuss a D-feature if its corresponding 
I-feature does not take a particular value.
To solve this problem, we design a new distance function. 
The mathematical proof of its soundness is provided in Appendix A.

In our solution, the distance wrt.\ I-features is characterised by a general distance function $f$, e.g. $p$-norm.
We only consider the special case of D-features. 
We define $\tau$ as a function that maps each D-feature to its corresponding I-feature, e.g. maps `labiodental' to `labial'. Consider $F_1$ and $F_2$, two feature vectors to be compared. 
Assume $j \in S_D$ is a D-feature, and $\tau(j) \in S_I$ is the corresponding I-feature.
\begin{equation}
    \setlen s_j\xlongequal{\text{def}}\sup\limits_{F_1,F_2\in\Omega}f(F_1^{j},F_2^{j})
\end{equation}
Denote the set of all valid feature vectors as $\Omega$, which is a subset of $\mathbb{R}^{14}$. We define a function $g_{j,\tau(j)}:\Omega \mapsto\mathbb{R}$ as follows:
\begin{equation}
\label{distance-func}
    \setlen
  g_{j,\tau(j)}(F_1, F_2)=c\cdot s_j+(1-c)f(F_1^{j},F_2^{j})
\end{equation}
where $c=\min\{f(F_1^{\tau(j)},F_2^{\tau(j)}),1\}$.
The intuition of the design of $g_{j,\tau(j)}$ is as follows. 
It is reasonable to compare $F_1^j$ and $F_2^j$ with a normal distance $f$, when the corresponding I-features $F_1^{\tau(j)}$ and $F_2^{\tau(j)}$ are equal (or very near, since they are continuous).
Otherwise, the distance between $F_1^j$ and $F_2^j$ should correspond to the maximum possible distance they can reach.

Now we are ready to define
\begin{equation} \label{eq:dist-func}
    d(F_1, F_2)=\sum_{k \in S_I}f(F_1^k, F_2^k)+\sum_{j \in S_D}g_{j, \tau(j)}(F_1, F_2)
\end{equation}

\subsection{The Restrictions} \label{sec:restriction}
To obtain a proper phonetic feature vector, we need to ensure the values of its D-features to be consistent with its corresponding I-features. 
When a D-feature is meaningless wrt.\ its I-feature, we force the D-feature's value to be near 0 by some mathematical tricks. 
Three cases are consider separately. 

\paragraph{Case I: delayed release} Unless the corresponding I-feature \textit{sonority} is around $1$, the value of the \textit{delayed release} feature is meaningless and thus should be around $0$. 
Therefore, the following constraint is considered:
\begin{equation} \label{constr-begin}
    \setlen
    F^{j} \leqslant \max(0, \min(F^{\tau(j)}, 2-F^{\tau(j)}))
    \end{equation}

\paragraph{Case II: high or front} Unless the corresponding I-feature \textit{dorsal} is around $1$, the value of a \textit{high} or \textit{front} feature should be around $0$. Ideally, the following constraints are satisfied:
    \begin{align}
        \setlen
        F^j \geqslant 1 \ &(\text{if }F^{\tau(j)} >0.5) \\
        F^j = 0   \ &(\text{if }F^{\tau(j)} \leqslant 0.5)
    \end{align}
    
To linearise, we define auxiliary variables $b$ (binary, indicator of whether $F^{\tau(j)}$ is larger than 0.5), $M$ (large enough), $\epsilon$ (small enough). We have:
\begin{align}\label{eq:highfront}
\setlen
    F^{\tau(j)} & \geqslant 0.5 + \epsilon - M \cdot (1 - b)\\
    F^{\tau(j)} & \leqslant 0.5 + M \cdot b\\
    \max & (0, 1-F^j)=1-b \label{M-trick}
\end{align}

\paragraph{Case III: Other D-features} When the value of the corresponding I-feature is around $1$, the absolute value of a D-feature $F^j$ should be around $1$.
Otherwise the absolute value should be close to $0$. 
We apply the same linearising trick, with only (\ref{M-trick}) changed into:
\begin{equation} \label{constr-end}
    \setlen
    |F^j|=b
\end{equation}

To sum up, our model is to minimise Eq. (\ref{eq:objective}) subject to constraints characterised by Eq. (\ref{constr-begin})--Eq. (\ref{constr-end}).

\definecolor{black}{RGB}{23,44,81} 
\definecolor{dark_blue}{RGB}{143,170,220} 
\definecolor{light_blue}{RGB}{218,227,243} 
\definecolor{lighter_blue}{RGB}{222,235,247} 
\definecolor{orange_yellow}{RGB}{255,192,0} 
\definecolor{darker_blue}{RGB}{68,114,196} 
\definecolor{orange_red}{RGB}{255,90,17} 

\begin{figure*}
    \centering
    \begin{tikzpicture}
     \tikzstyle{every node}=[font=\footnotesize]
    \draw[rounded corners] (-0.5,0.3) rectangle (6,4.3);
    \node[align=left,anchor=north west] at (-0.5, 4.3){\textbf{\normalsize Step 1: Simulating a consonant}\\\textbf{ \normalsize system}};
    \draw  (-0.3,0.8) rectangle (0.7,2.8);
    \node[align=left,anchor=south west] at (-0.2,1.4) {IPA\\chart};
    
     \draw [black, line width=0.1cm, arrows = {-Stealth[length=3mm]}]  (0.75,1.8) -- (1.9,1.8);
    \node[align=left,anchor=south west] at (0.65,1.8) {{\scriptsize randomly}\\{\scriptsize sample}};

    \draw [rounded corners]  (1.9,0.4) rectangle (5.9,3.6);
    
    \node[align=left,anchor=north west] at (1.9,3.6) {\textbf{Selected}\\ \textbf{Initials}};
    \node[align=left,anchor=north west] at (3.3,3.6) {\textbf{Characters}};

    \node[align=left,anchor=north west] at (1.9,2.8) {$I_1$: dz\textipa{\super{w}}};
    \node[align=left,anchor=north west] at (1.9,2.4) {$I_2$: k\textipa{\super{h}}};
    \node[align=left,anchor=north west] at (1.9,2) {$I_3$: x};
    \node[align=center,anchor=north west] at (1.9,1.5) {$\cdots$};
    \node[align=left,anchor=north west] at (1.9,1.1) {$I_m$:\textipa{\textlambda}};

    \node[align=left,anchor=north west] at (3,2.8) {$C_{1_1}, C_{1_2}\cdots C_{1_{n_1}}$};
    \node[align=left,anchor=north west] at (3,2.4) {$C_{2_1}, C_{2_2}\cdots C_{2_{n_2}}$};
    \node[align=left,anchor=north west] at (3,2) {$C_{3_1}, C_{3_2}\cdots C_{3_{n_3}}$};
    \node[align=center,anchor=north west] at (3.3,1.5) {$\cdots$};
    \node[align=left,anchor=north west] at (3,1.1) {$C_{m_1},C_{m_2} \cdots C_{m_{n_m}}$};

    \draw [orange_yellow, line width=0.04cm]  (1.9,2.3) rectangle (5.7,2.8);

    \draw [black, line width=0.1cm, arrows = {-Latex[length=3mm]}]  (6,2.3) -- (6.5,2.3) -- (6.5,3.75) -- (7,3.75);
    \draw [black, line width=0.1cm, arrows = {-Latex[length=3mm]}] (6.5,2.3) -- (6.5,1.05) -- (7,1.05);

    \draw [rounded corners] (7,2.5) rectangle (15,5);
    \node[align=left,anchor=north west] at (7, 5){\textbf{\normalsize Step 2: Deriving character-speller pairs}};
    \draw(7.2,2.6) rectangle (10.8,4.4);
    \node[align=left,anchor=north west] at (7.2, 4.4){\scriptsize (character, \textcolor{darker_blue}{speller})};
    \node[align=left,anchor=west] at (7.2, 3.4){$I_1$};
    \draw [decorate, line width=0.04cm,
    decoration = {calligraphic brace, 
        raise=-1pt, 
        aspect=0.5, 
        amplitude=3pt 
    }] (7.7, 2.9) -- (7.7, 3.9);
    
    \node[align=left,anchor=west] at (7.7, 3.8){\scriptsize $C_{1_i}$, \textcolor{darker_blue}{$C_{1_{i^{\prime}}}$}};
    \node[align=left,anchor=west] at (9, 3.8){w.p. $1-p_{\text{fq}}$};
    \node[align=left,anchor=west] at (7.7, 3.4){\scriptsize $C_{1_j}$, \textcolor{darker_blue}{$C_{k_{j^{\prime}}}$}};
    \node[align=left,anchor=west] at (9, 3.4){w.p. $p_{\text{fq}}$};
    \node[align=left,anchor=west] at (7.7, 3){\scriptsize $\cdots \quad \cdots$};
    \draw [orange_yellow, line width=0.04cm]  (7.3,2.7) rectangle (10.7,4);
    
    \draw(11.2,2.6) rectangle (14.8,4.4);
    \node[align=left,anchor=north west] at (11.2, 4.4){\scriptsize (character, \textcolor{darker_blue}{speller})};
    \node[align=left,anchor=west] at (11.15, 3.4){$I_k$};
    \draw [ decorate, line width=0.04cm,
    decoration = {calligraphic brace, 
        raise=-1pt, 
        aspect=0.5, 
        amplitude=3pt 
    }] (11.7,2.9) -- (11.7, 3.9);
    
    \node[align=left,anchor=west] at (11.7, 3.8){\scriptsize $C_{k_i}$, \textcolor{darker_blue}{$C_{k_{i^{\prime}}}\quad $}};
    \node[align=left,anchor=west] at (13.1, 3.8){w.p. $1-p_{\text{fq}}$};
    \node[align=left,anchor=west] at (11.7, 3.4){\scriptsize $C_{k_j}$, \textcolor{darker_blue}{$C_{k_{j^{\prime}}}$}};
    \node[align=left,anchor=west] at (13.1, 3.4){w.p. $1-p_{\text{fq}}$};
    \node[align=left,anchor=west] at (11.7, 3){\scriptsize $\cdots \quad \cdots$};
    
    \node[align=center,anchor=north] at (11,3.7) {$\mathbf{\cdots}$};
    \draw[black, arrows = {-Latex[scale=1]}, line width=0.04cm] (11.5,3.2) arc (360:182:1.4 and 0.6);

    \draw [rounded corners] (7,-0.4) rectangle (15.5,2.5);
    \node[align=left,anchor=north west] at (7, 2.5){\textbf{\normalsize Step 3: Generating variations}};
    
    \draw(7.2,-0.3) rectangle (12,1.9);
    \node[align=left,anchor=north west] at (7.2, 1.9){\textbf{Variety 1:}};
    \node[align=left,anchor=west] at (7.2, 1.2){$I_1\colon$dz\textipa{\super{w}}};
    \draw [orange_yellow, line width=0.04cm]  (7.3, 1) rectangle (8.35,1.4);
    
    \draw [darker_blue, line width=1pt, double distance=5pt, arrows = {-Latex[length=2mm, scale width =2.5]}] (8.4,1.2) -- (9.2,1.2);
    \node[align=left,anchor=west] at (8.35,1.2){$p_{\text{dia}}$};
    
    \draw [black, line width=2.5pt, arrows = {-Latex[length=1mm, scale width=2]}] (8.7,0.7) -- (8.7,1.1);
    \node[align=left,anchor=west] at (8.35,1.2){$p_{\text{dia}}$};
    \draw (8.7,0.5) ellipse (0.37 and 0.2);
    \node[align=center,anchor=west] at (8.25,0.5){noise};
    
    \node[align=left,anchor=west] at (9.15, 1.2){dz};
    
    \draw [decorate, line width=0.04cm,
    decoration = {calligraphic brace, 
        raise=-2pt, 
        aspect=0.71, 
        amplitude=3pt 
    }] (9.7,0) -- (9.7, 1.7);
    
    \node[align=left,anchor=north west] at (9.7,1.9){$C_{1_1}$};
    \node[align=left,anchor=north west] at (9.7,1.5){$C_{1_2}$};
    \node[align=left,anchor=north west] at (9.7,1.1){$C_{1_3}$};
    \node[align=left,anchor=north west] at (9.7,0.7){$\cdots$};
    \node[align=left,anchor=north west] at (9.7,0.4){$C_{1_{n_1}}$};

    \draw [darker_blue, line width=1pt, double distance=5pt, arrows = {-Latex[length=2mm, scale width =2.5]}] (10.5, 1.65) -- (11.4, 1.65);
    \node[align=left,anchor=west] at (10.45, 1.65){$p_{\text{char}}$};
    
    \draw [darker_blue, line width=1pt, double distance=5pt, arrows = {-Latex[length=2mm, scale width =2.5]}] (10.5,0.1) -- (11.4,0.1);
    \node[align=left,anchor=west] at (10.45,0.1){$p_{\text{char}}$};
    
    \draw [black, line width=2.5pt, arrows = {-Latex[length=1mm, scale width=2]}] (11,1.1) -- (11,1.5);
    \draw [black, line width=2.5pt, arrows = {-Latex[length=1mm, scale width=2]}] (11,0.6) -- (11,0.2);
    \draw (11,0.85) ellipse (0.37 and 0.2);
    \node[align=center,anchor=west] at (10.55,0.85){noise};

    \node[align=left,anchor=north west] at (11.4,1.9){ts};
    \node[align=left,anchor=north west] at (11.4,1.5){dz};
    \node[align=left,anchor=north west] at (11.4,1.1){dz};
    \node[align=left,anchor=north west] at (11.4, 0.7){$\cdots$};
    \node[align=left,anchor=north west] at (11.4, 0.4){dz\textipa{\super{j}}};

    \node[align=left,anchor=north west] at (11.9, 1.1){$\cdots$};
    
    \draw(12.5,-0.3) rectangle (15.3,1.9);
    \node[align=left,anchor=north west] at (12.5, 1.9){\textbf{Variety 20:}};
    \node[align=left,anchor=west] at (12.5, 1.2){$I_1\colon$dz\textipa{\super{w}}};

    \draw [darker_blue, line width=1pt, double distance=5pt, arrows = {-Latex[length=2mm, scale width =2.5]}] (14.4,1.2) -- (15.3,1.2);
    \node[align=left,anchor=west] at (14.35,1.2){$p_{\text{char}}$};

    \draw [darker_blue, line width=1pt, double distance=5pt, arrows = {-Latex[length=2mm, scale width =2.5]}] (13.65,1.2) -- (14.45,1.2);
    \node[align=left,anchor=west] at (13.6,1.2){$p_{\text{dia}}$};
    
    \draw [orange_yellow, line width=0.04cm]  (12.6, 1) rectangle (13.65,1.4);

    \draw [black, line width=2.5pt, arrows = {-Latex[length=1mm, scale width=2]}] (14.05,0.7) -- (14.05,1.1);
    
    \draw (14.05,0.5) ellipse (0.37 and 0.2);
    \node[align=center,anchor=west] at (13.6,0.5){noise};
    
    \draw [black, line width=2.5pt, arrows = {-Latex[length=1mm, scale width=2]}] (14.9,0.7) -- (14.9,1.1);
    \draw (14.9,0.5) ellipse (0.37 and 0.2);
    \node[align=center,anchor=west] at (14.45,0.5){noise};

    \node[align=left,anchor=north west] at (12.5, 0.4){$I_2\colon$k\textipa{\super{h}} $\quad \cdots $};
    
    \end{tikzpicture}
    \caption{Overview of synthesis data.} \label{fig:synthesis-flow}

\end{figure*}
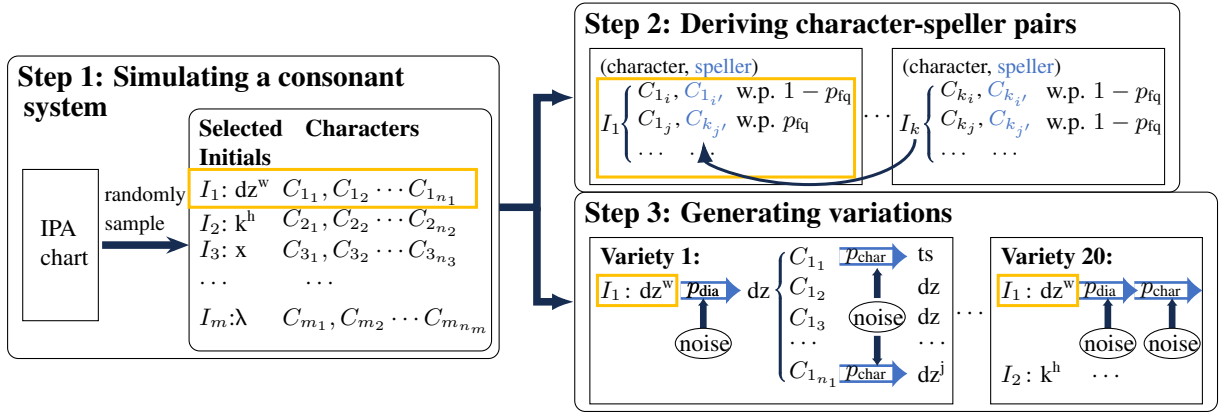


    
\begin{table*}[thbp]
    \centering
    \scalebox{0.9}{
    \begin{tabular}{rrrrl}
    \toprule
    \multicolumn{1}{l}{\textbf{Example}} & \multicolumn{1}{l}{\textbf{I-feature $\tau(j)$}} & \multicolumn{1}{l}{\textbf{D-feature $j$}} & \multicolumn{1}{l}{\textbf{Valid Combinations}} & \textbf{Shortest Distance} \\
    \midrule
    \multicolumn{1}{l}{(1.048, 0.905)} & \multicolumn{1}{l}{sonority} & \multicolumn{1}{l}{delayed release} & \multicolumn{1}{l}{\textbf{(1, 1)}, (1, -1), (2/3/4/5, 0), (0, 0)} & 0.143 \\
    \multicolumn{1}{l}{(0.946, -0.919)} & \multicolumn{1}{l}{labial} & \multicolumn{1}{l}{labiodental} & \multicolumn{1}{l}{(1, 1), \textbf{(1, -1)}, (-1, 0), (0, 0)} & 0.135 \\
    \multicolumn{1}{l}{(0.499, 0.988)} & \multicolumn{1}{l}{coronal} & \multicolumn{1}{l}{anterior} & \multicolumn{1}{l}{\textbf{(1, 1)}, (1, -1), (-1, 0), (0, 0)} & 0.503 \\
    \multicolumn{1}{l}{(0.499, 0.499)} & \multicolumn{1}{l}{coronal} & \multicolumn{1}{l}{distributed} & \multicolumn{1}{l}{(1, 1), (1, -1), (-1, 0), \textbf{(0, 0)}} & 0.998 \\
    \multicolumn{1}{l}{(0.992, 1.952)} & \multicolumn{1}{l}{dorsal} & \multicolumn{1}{l}{high} & \multicolumn{1}{l}{(-1, 0), (\textbf{1}, 1/\textbf{2}/3), (0, 0)} & 0.056 \\
    \multicolumn{1}{l}{(0.992, 2.889)} & \multicolumn{1}{l}{dorsal} & \multicolumn{1}{l}{front} & \multicolumn{1}{l}{(-1, 0), (\textbf{1}, 1/2/\textbf{3}), (0, 0)} & 0.119 \\
    \midrule
          &       &       & \textbf{Total Distance:} & \textbf{1.954} \\
    \bottomrule
    \end{tabular}%
    }
\caption{Demonstration of how to calculate the `total distance' from a reconstructed vector. The `Example' column contains all the ($\tau(j), j$) value pairs in a reconstructed vector. 
For each pair, we highlight the shortest $L_1$ distance between it and all valid combinations with bold font. 
`Total Distance' is the sum of all shortest distances over $j$.}
\label{tab:self-sound}
\end{table*}

\section{Validation Experiments on Synthetic Data}
\label{sec:synthesis}

Since language reconstruction lacks a definitive ground truth, it is challenging to discuss the `correctness' of any reconstruction result.
We validate our method on a wide range of synthetic datasets, which hopefully mirror diachronic phonetic change. 
Starting from a predefined consonant system, we create varieties of it by introducing systematic change and random noise. 
We then extract character--speller pairs to mimic the \fq~information.
In order to evaluate the effectiveness in reconstructing the predefined consonant system, we apply our model to the varieties as well as character--speller data.
The idea of experimentation with synthetic data has been utilized to simulate lexical semantic change (\citealp{rosenfeld-erk-2018-deep,shoemark-etal-2019-room}).

\subsection{Generating Synthetic Data} \label{sec:generate-synthesis}
Data synthesis is demonstrated in Figure \ref{fig:synthesis-flow}.
It consists of three steps, as explained as follows.

\paragraph{Step 1: Selecting/simulating a consonant system}
To generate the old stage initial system, we randomly sample an initial set $S_I=\{I_1, I_2, \dots, I_m\}$ from an IPA chart of consonants
\footnote{The IPA chart is based on Hayes's feature spreadsheet (\url{https://brucehayes.org/120a/index.htm\#features}). 
Diacritics [\textipa{\super h} \textipa{\super w} \textipa{\super j}] are additionally considered.
}
, namely $S_{\text{IPA}}$, with $m\in[35,40]$. For each initial $I_i (1 \leq i \leq m)$, we generate a set of characters $S_{C_i}=\{C_{i_1}, C_{i_2}, \dots, C_{i_{n_i}}\}$, with the number of characters $n_i$ falling within $[20,80]$. 

In addition to purely artificial consonant system, we also utilize 
modern English, German, Mandarin and reconstructed Latin\footnote{
See \url{https://en.wikipedia.org/wiki/English_phonology}, 
\url{https://en.wikipedia.org/wiki/Standard_German_phonology},
\url{https://en.wikipedia.org/wiki/Standard_Chinese_phonology}, and \url{https://en.wikipedia.org/wiki/Latin_phonology_and_orthography} respectively.}.

\paragraph{Step 2: Deriving character--speller pairs} 
Following the model of rhyme dictionaries, we assign an artificial \fq~ spelling to each character. 
Given that not all characters and their spellers share the same initial (\S\ref{sec:obj}), we introduce variability by randomly assigning a portion $p_{\text{fq}}$ of characters to have their upper spellers randomly selected from $S_I$.


\paragraph{Step 3: Generating variations} 
We generate 20 varieties based on $S_I$ and $S_{C_i}$s to simulate sound change, denoted as $S_I^v=\{I_1^v, I_2^v, \dots, I_m^v\}$ and $S_{C_i}^v=\{C_{i_1}^v, C_{i_2}^v, \dots, C_{i_{n_i}}^v\} (1\leq i \leq m, 1\leq v \leq 20)$.
We assume that most sound changes are regular, where the phonetic value of an initial influences all characters with that initial in a given variety. To simulate regular sound change, initial $I_i$ can change to any $I_i^v \in S_\text{IPA}$ in variety $v$ with probability $p_\text{dia}$.\\
Exceptions to regular change can occur due to various causes, e.g. lexical borrowing and grammatical analogy, and we model such irregular change by allowing the initial of character $C_{i_j}^v$ to change from $I_i^v$ to any consonant in $S_\text{IPA}$ with probability $p_\text{char}$.\\
Denote the $L_1$ distance between $I_a$ and $I_b$ (both in $S_\text{IPA}$) as $d_{I_a, I_b}$.

\begin{figure*}[t!]
  \centering
    \subfigure[$p_{\text{dia}}=0.3$]{
    \label{Fig:baseline-Latin-1}
    \includegraphics[width=0.49\linewidth]{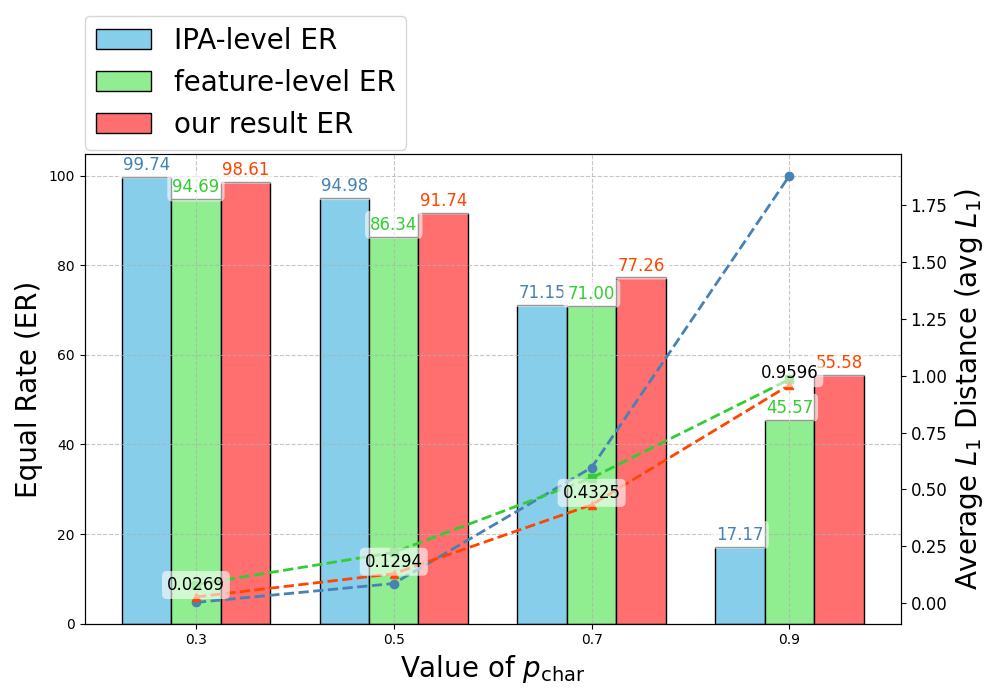}}
    \subfigure[$p_{\text{dia}}=0.5$]{
    \label{Fig:baseline-Latin-2}
    \includegraphics[width=0.49\linewidth]{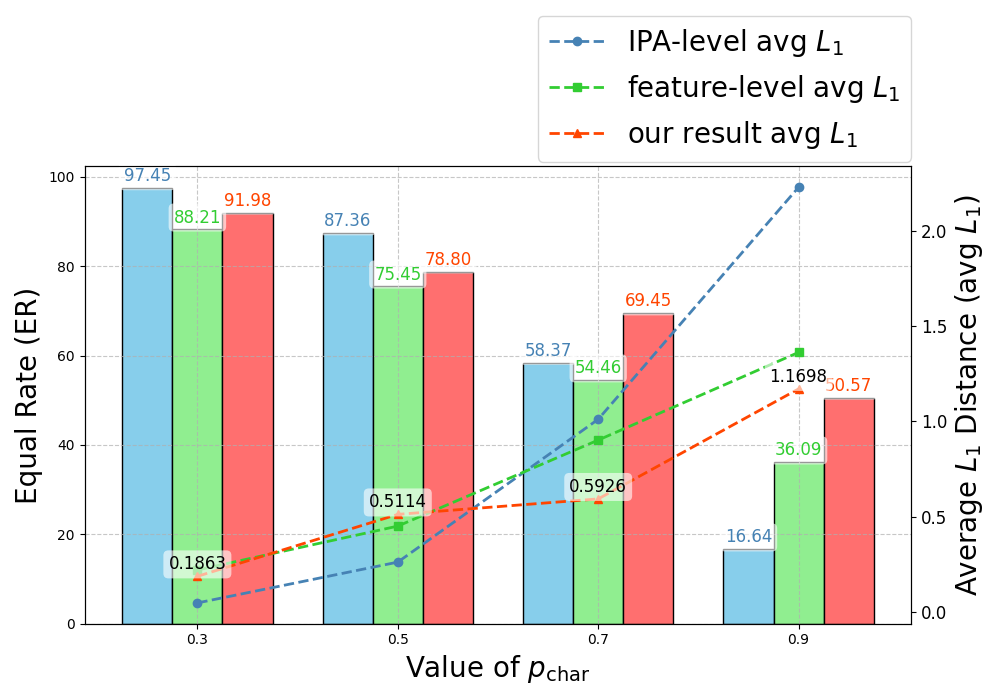}}
    \subfigure[$p_{\text{dia}}=0.7$]{
    \label{Fig:baseline-Latin-3}
    \includegraphics[width=0.49\linewidth]{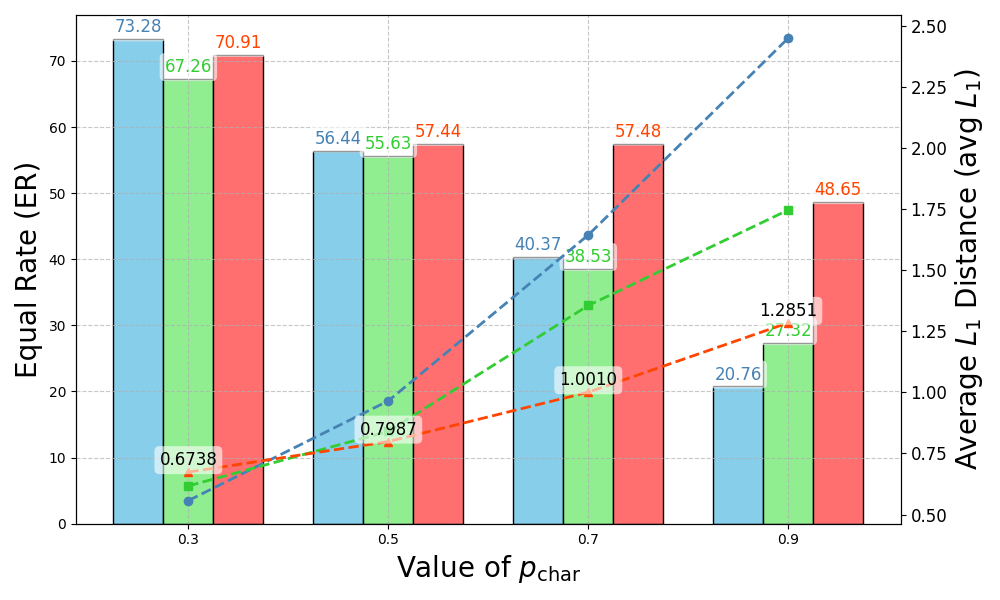}}
    \subfigure[$p_{\text{dia}}=0.9$]{
    \label{Fig:Fig:baseline-Latin-4}
    \includegraphics[width=0.49\linewidth]{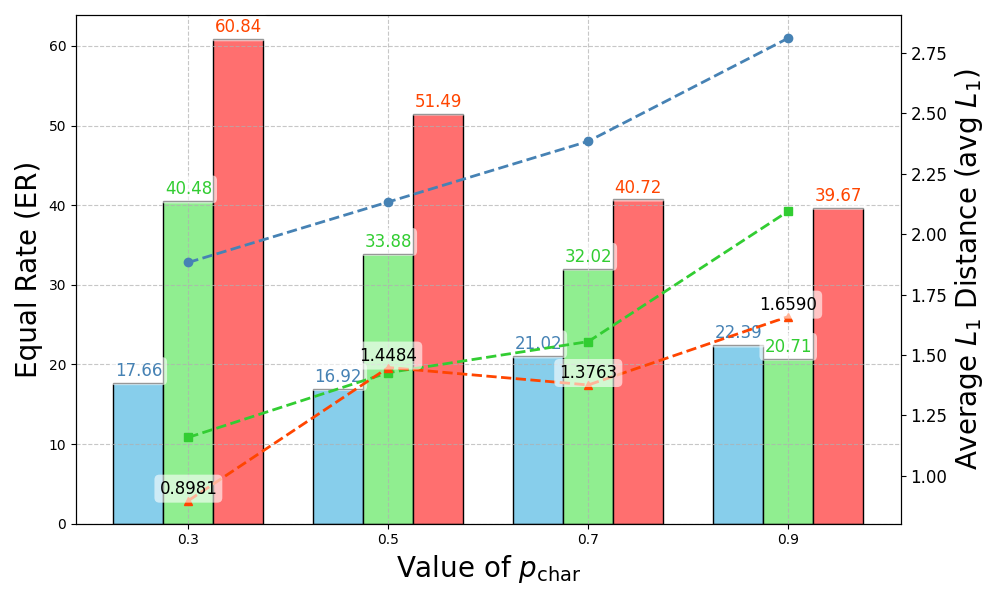}}
    \caption{Comparison between our results and baselines with synthetic data that starts from Latin consonant system, with respect to equal rate (ER) and average $L_1$ distance (avg $L_1$).}
  \label{Fig:baseline-Latin}
\end{figure*}

\subsection{Experimental Setup} \label{sec:synthesis-setup}
We use the Gurobi\footnote{\url{https://www.gurobi.com}} MIP solver for our empirical investigation.
We set \texttt{MIPGap} to 1e-4, and the \texttt{TimeLimit} to 8 hours as the maximal time for calculation. 
Notably, the obtained solutions usually are not optimal. 
Nevertheless, they are of relatively good quality to verify the reliability of our model.


We consider the following three metrics to evaluate the goodness of a reconstruction result, when ground-truth is available.
The ground-truth consonant system of the experiments in this section is either stochastically sampled from IPA, the reconstructed Latin, or modern English, German and Mandarin.

\paragraph{Average $L_1$} Since we represent phonemes by vectors, a straightforward way to evaluate the goodness of reconstruction is to calculate the overall distance between reconstructed vectors and their corresponding ground-truth vectors. 
To this end, we report the average $L_1$ distance. 

\paragraph{Equal rate} A more strict evaluation metric is to reward only when the reconstructed vector is extremely close to its ground-truth. 
Here, we consider a phoneme as successfully reconstructed only when the $L_1$ distance between the reconstructed vector and its predefined value is smaller than $10^{-4}$,  
Accordingly, we report  the proportion of successfully constructed initials as \textbf{equal rate}.

\paragraph{Soundness of phonetic feature vector} 
The reconstructed results should be valid phonemes that satisfy the constraints on D-features listed in \S\ref{sec:restriction}. 
Our features are continuous, and we thus propose to measure their deviation from the constraints rather than classifying them as strictly `valid' or `invalid'.
For each D-feature $j$ and its corresponding I-feature $\tau(j)$, we consider the shortest $L_1$ distance between our result and all the valid values of $(j, \tau(j))$.
Table \ref{tab:self-sound} serves as an example, listing all $(j, \tau(j))$ pairs and their valid combinations. 
We report \textbf{sound rate}---the proportion of characters with a total distance less than $10^{-4}$.


\subsection{Results and Analysis} \label{sec:synthesis-result-analysis}
Our main results are shown in Figure \ref{Fig:baseline-Latin}, where we compare our results with two baselines. Since the phenomenon of characters having different initials from their upper spellers is not common in real data, we set $p_{\text{fq}}=0.1$.

We report two versions of majority vote results as baseline: IPA-level and feature-level. Considering the randomness in generating consonant systems and their variations, we conduct the experiment three times and report the average for each setting. In the IPA-level majority vote, for each character, we select the most frequent IPA phoneme from all 20 dialects to reconstruct its initial. For feature-level voting, we choose the most frequent value for each feature of each character. 

The IPA-level majority vote achieves the highest equal rate when $p_{\text{dia}}$ and $p_{\text{char}}$ are small, but its equal rate declines rapidly as randomness increases. In contrast, the feature-level majority vote performs better under high $p_{\text{dia}}$ and $p_{\text{char}}$ settings. 
Compared with the baselines, our model significantly outperforms both in terms of equal rate and average $L_1$ distance across most settings, particularly when $p_{\text{dia}}$ and $p_{\text{char}}$ are large, which highlights the robustness of our model.

To estimate a proper value of the change rate $p_{\text{dia}}$ is challenging.
We use the following geometric method to estimate a lower bound of $p_{\text{dia}}$ from the ancestral form (MC) to modern dialects based on the differences among all modern dialects. 
First, we measure how different any two dialects are by calculating the percentage of characters with different initial pronunciation. For example, this proportion between Beijing and Guangzhou is 65.70\%. 
Intuitively, halving such a difference degree gives a lower bound though the estimation based on any single pair of dialects should be far from being tight. 
We then leverage the concept of high-dimensional sphere to integrate all possible pairs of dialects.
The key idea is as follows. 
Each dialect is viewed as a point in a high-dimensional space and the percentage of characters with different pronunciation of initials is viewed as the distance between the corresponding pair of dialects. 
It is easy to see that such a distance measurement satisfies the triangle inequality. 
The radius of the minimal high-dimensional sphere that covers all dialects serves as a (loose) lower bound. 
Based on the data from \citeauthor{zihui}, 
we first convert the distance matrix into coordinates using the algorithm proposed by \citet{calculation-of-coordinates}, 
then apply the algorithm proposed by \citet{Smallest-Enclosing-Ball} to determine the radius of the minimal sphere, obtaining an empirical value of 0.4180. 
The maximum of such lower bound is 0.8844, when any two dialects have totally different pronunciation (in other words, the distance is 1).
The empirical result suggests to utilise a high ratio of sound change.



Since the vectors derived from feature-level voting do not necessarily correspond to valid phonemes, we compare its sound rate with that of our model. Although our model maintains a high sound rate across all settings, we perform a two-proportion z-test to determine whether the difference between our model and the feature-level majority vote in terms of SR is statistically significant. The null hypothesis is that the SR of our method is equal to that of the baseline. The results are reported in Table \ref{tab:hypo-test}. If we test this hypothesis at a significance level of 95\%, for all settings except (0.5, 0.9), the statistic $z$
exceeds $z_{0.975}=1.96$, and we can reject the null hypothesis. For the (0.5, 0.9) setting, $z=1.7152$ is still larger than $z_{0.95}=1.65$. 
Therefore, we conclude that our model performs better than the baselines in terms of sound rate.

\begin{table}[htbp]
  \centering
  \scalebox{0.9}{
    \begin{tabular}{cccccc}
    \toprule
    \textbf{Setting} & \textbf{SR$_1$} & \textbf{$n_1$} & \textbf{SR$_2$} & \textbf{$n_2$} & \textbf{stat $z$} \\
    \midrule
    (.3, .3) & 1.0000  & 1078  & 0.9733  & 3138  & 5.4191  \\
    (.3, .5) & 0.9847  & 1110  & 0.9515  & 3290  & 4.8737  \\
    (.3, .7) & 0.9991  & 1107  & 0.9338  & 3436  & 8.6460  \\
    (.3, .9) & 0.9950  & 1001  & 0.9586  & 3160  & 5.6477  \\
    (.5, .3) & 0.9872  & 1173  & 0.9579  & 3228  & 4.7228  \\
    (.5, .5) & 0.9943  & 1047  & 0.9558  & 3109  & 5.9036  \\
    (.5, .7) & 0.9895  & 954   & 0.9422  & 3216  & 6.0635  \\
    (.5, .9) & 0.9826  & 1037  & 0.9731  & 3158  & 1.7152  \\
    (.7, .3) & 0.9829  & 994   & 0.9671  & 3108  & 2.5809  \\
    (.7, .5) & 0.9804  & 1172  & 0.9245  & 3174  & 6.8637  \\
    (.7, .7) & 0.9942  & 1030  & 0.9389  & 3165  & 7.2458  \\
    (.7, .9) & 0.9889  & 1169  & 0.9247  & 2985  & 8.0104  \\
    (.9, .3) & 0.9865  & 1040  & 0.9678  & 3190  & 3.1966  \\
    (.9, .5) & 0.9952  & 1048  & 0.9424  & 3179  & 7.1880  \\
    (.9, .7) & 0.9904  & 1038  & 0.9315  & 3272  & 7.2954  \\
    (.9, .9) & 0.9873  & 1004  & 0.9141  & 3339  & 8.0252  \\
    \bottomrule
    \end{tabular}%
    }
  \caption{Two-proportion z-test between our result and feature-level majority vote with respect to sound rate (SR). SR$_1$ and SR$_2$ represent the sound rates of our model and the baseline, respectively, while $n_1$ and $n_2$ indicate the sample sizes (number of characters) of our model and the baseline.}
  \label{tab:hypo-test}%
\end{table}%


Though remarkable, we should exercise caution when interpreting the results. 
Naturally occuring sound changes display greater regularity in some cases but are much less regular in others.



\begin{table*}[thbp]
  \centering
    \begin{tabular}{cccccc}
    \toprule
     & \multicolumn{4}{c}{$k=1$}&\multicolumn{1}{c}{$k=3$}\\
    \cmidrule(lr){2-5} \cmidrule(lr){6-6} 
    \multicolumn{1}{c}{\textbf{Setting}} & $\lambda_{\text{fq}}=0$ & $\lambda_{\text{fq}}=0.5$ & $\lambda_{\text{fq}}=0.75$ & $\lambda_{\text{fq}}=0.95$ & \multicolumn{1}{l}{$\lambda_{\text{fq}}=0.5$} \\
    \midrule
    (0.1, 0.3, 0.3) &  96.85\%&  \textbf{98.61}\%& 98.89\% & 90.35\% & 98.70\% \\
    (0.1, 0.5, 0.5) & 75.45\% & \textbf{78.80}\% & 79.94\% & 68.77\% & 79.27\% \\
    (0.1, 0.7, 0.7) & 54.37\% & \textbf{57.48}\%& 59.13\% & 53.69\% & 57.67\% \\
    \bottomrule
    \end{tabular}%
    \caption{\label{tab:synthesis-adjust-parameters}Results on Latin consonant system with parameters adjusted. 
    The numbers in the `Settings' column correspond to ($p_{\text{fq}}, p_{\text{dia}}, p_{\text{char}}$) respectively. 
    $k$ represents the weight assigned to character and speller pairs with matching medials, as defined in \S\ref{sec:model-context}.} 
\end{table*}

\paragraph{Base Distance Function $f$} \label{form-of-function}
The general distance function $f$ (defined in \S\ref{sec:obj}) is also a hyperparameter. 
In Figure \ref{Fig:baseline-Latin}, it is set as $f(x_1,x_2)=|x_1-x_2|$. 
We did a number of auxiliary experiments with different $p$-norm distance functions. 
Even the quadratic function significantly increases the difficulty to the corresponding optimisation problems, without substantial improvement in performance. 
It seems that the only practical option for $f$ is $L_1$.
All following experiments are based on this choice.
%

\paragraph{Weight of \fq~($\lambda_{\text{fq}}$)} \label{sec:fq-weight}
We adjust the weight of terms related to \fq~in the objective function, i.e. $\lambda_{\text{fq}}$ in Eq. (\ref{eq:objective}). By default, $\lambda_{\text{fq}}$ is set to 0.5, and we explore its effect with respect to equal rate in Table \ref{tab:synthesis-adjust-parameters}. Setting $\lambda_{\text{fq}}$ to 0 (i.e., not using \fq~information) results in a decrease in the equal rate, while increasing it to 0.75 improves the equal rate. However, further increasing it to 0.95 leads to a decline. These experiments suggest that the choice of $\lambda_{\text{fq}}$ is empirical, and we will adjust it accordingly when working with real data (\S\ref{sec:eval-AMI}).

\paragraph{Results with English, German and Mandarin}
To evaluate the robustness, we experiment with synthetic data that starts from a consonant system in natural phonology. We choose modern standard English, German and Mandarin as representatives. 
We also present results from random consonant system for comparison.
The results are in Table \ref{tab:other-language}, under the setting of (0.1, 0.5, 0.3). The remarkable results reaffirm the reliability of our model.

It is worth noting that reconstructing natural consonant systems is much easier than artificial ones.

\begin{table}[htbp]
  \centering
\begin{tabular}{ccccc}
    \toprule
      & \textbf{Ger} & \textbf{Man} & \textbf{Eng} & \textbf{RND} \\
      \midrule
\textbf{ER}(\%) & 96.10 &  94.48   &  93.02  & 84.73 \\
\textbf{Avg. $L1$}\unboldmath{} & .1894 &  .0884    & .1519 & .2549 \\
    \bottomrule
\end{tabular}%
    \caption{\label{tab:other-language}Results with synthetic data that starts from German, Mandarin, English, and the random system.}
\end{table}%

\subsection{Modeling the Influence of Context} \label{sec:model-context}
The phonetic value of phonemes can be influenced by its context. 
For initials in Chinese syllables, the major influential context is the medial that follows. 
To integrate medial information in MC\footnote{The data is provided by Peking University.}
into our model, we adjust the weight of terms related to \fq~in the objective function. 
For each character and speller pair, i.e. $(X,X_u) \in S_{\text{fq}}$, we assign a weight of $k$ to $d(F_{\text{MC}}(X), F_{\text{MC}}(X_u))$ if they share the same medial, otherwise $1$. In the basic setting, $k=1$, and increasing it aims to improve the likelihood of pairs with matching medials sharing the same initial.

Setting $k$ to 3, we present representative results in Table \ref{tab:synthesis-adjust-parameters}. Changing $k$ from 1 to 3 has little influence on equal rate, indicating that medial information has already been well captured. 
Notably, our model is inherently conditional, 
as we always consider an initial in a particular character, where the medial, main vowel and coda are all fixed. 
When human scholars encounter overlapping heterogeneous information sources, decision-making becomes challenging, while our model provides a possible technique for such challenging issues.

The results demonstrate the flexibility of our model---heterogeneous information can be seamlessly integrated as constraints or terms in the objective function, and be integrated into our model conveniently. Furthermore, the flexibility to adjust weights of different terms allows fine-tuning according to specific requirements.

\section{Validation Experiments on Real Data} 
\label{sec:experiment}

\subsection{Collecting Real Data} \label{sec:dataset}

We examine our model with the spelling information in \qy~and the phonetic information of 20 dialects in \citet{zihui}.
Polyphonic characters (characters with multiple pronunciations) are common in both MC and dialects. 
We treat different pronunciations of the same character as different entries and correlate different vectors to them. 
The final dataset consists of 1960 different characters and 2661 entries in total\footnote{We will release the dataset for research.}.

\paragraph{The \qy~information}
Only fragments of the original \qy~survived, and most commonly used documents are its revisions. 
The most accurate revision is \gy~廣韻. 
Though published in the Song Dynasty, \gy~is commonly believed to record the \qy~system and reflect the status of MC. 
\gy~was heavily used in traditional philological research, including \citet{gbh} and \citet{wangli-1957}.
We collect and integrate information from two electronic versions of \gy, separately provided by Peking University and  Beijing Normal University.

\paragraph{The dialect information}
\citet{zihui} is a workbook for fieldwork on Chinese dialects.
There is information for 20 modern Chinese dialects: Beijing, Jinan, Xi'an, Taiyuan, Wuhan, Chengdu, Hefei, Yangzhou (Mandarin), Suzhou, Wenzhou (Wu), Changsha, Shuangfeng (Xiang), Nanchang (Gan), Meixian (Hakka), Guangzhou, Yangjiang (Yue), Xiamen, Chaozhou, Fuzhou, and Jianou (Min). 
For each of these dialects, it documents both the phonological system and the phonetic values of representative characters.

\paragraph{Selecting representative characters} 
The original \qy~dataset has 25333 entries, but a large proportion of them are rarely used. 
In contrast, \citet{zihui} contains less than 3000 frequently used characters.
We denote the characters included in \gy~and \citet{zihui} as $S_{\text{gy}}$ and $S_{\text{zh}}$, respectively, and denote all characters used as \fq~spellers of characters in \gy~as $S_{\text{fq}}$ ($|S_{\text{fq}}|=$1462).
Instead of using all available characters, we aim to select a set of representative characters that comprehensively reflect the entire phonological systems. Our selection process involves the following steps:
\begin{enumerate}
    \item[1.] Subtract a smaller set $S_{*}$ from $S_{\text{gy}}$ for subsequent selection. Since \fq~spellings connect different characters and encapsulate valuable relationships between them, they are essential for deriving phonological categories. Therefore, We define $S_{\cap}$ to include common characters as well as \fq~spellers. Specifically, $S_{\cap}=S_{\text{gy}} \cap (S_{\text{fq}} \cup S_{\text{zh}})$ with 3990 entries. 
    \item[2.] For each character in $S_{\cap}$, if both its upper and lower spellers are also in $S_{\cap}$, this character is considered of particular interest. This set is denoted as $S^{*}$, which contains 2461 entries.
    \item[3.] Among $S^{*}$, if several characters share the same \fq, indicating that they are homophones, we select the first character only, which is often the most frequent character. We denote this set as $S^{*}_{1}$.
    \item[4.] Finally, we include \fq~spellers themselves into the selected set to link different entries. Our final representative character set is $S^{*}_{1} \cup (S_{\text{fq}} \cap S_{\cap})$, with 2661 entries.
\end{enumerate}



\subsection{Results and Analysis} \label{sec:real-data-evaluate}
 
\paragraph{Matching to held-out \fq~data}
Ideally, each character should share the same initial with its \fq/\zhiyin~speller. 
We randomly take 70\% of \fq/\zhiyin~material for MC reconstruction, and use the remaining 30\% for evaluation. 
We consider a character--speller pair as having matching initials if the $L_2$ distance between a character's reconstructed initial vector and that of its upper speller's is smaller than $10^{-4}$. 
We report the average $L_2$ distance between the reconstructed initials in character--speller pairs and the rate of pairs with matching initials as the \textbf{matching rate}.

The results are shown in Table \ref{table:held-out}. 
A large portion of held-out character--speller pairs have matching reconstruction, affirming the self-consistency of our results. 

\begin{table}[htbp] 
  \centering
    \begin{tabular*}{\linewidth}{@{}cccc@{}}
        \toprule      {$\lambda_{\text{fq}}$} & 0.5 &0.75 &0.95 \\
        \midrule
        \textbf{Matching Rate} & 66.38\% &  65.33\%    & 67.96\% \\
        \textbf{Avg. $L_2$} &  1.1062  &  1.1046  & 1.2432 \\
        
        \bottomrule
        \end{tabular*}
    \caption{\label{table:held-out}Evaluation with held-out \fq. 
    } 
\end{table}%

\section{Reconstruction Results and Discussion} 
\label{sec:main-result}
We obtain our final reconstruction result by applying the method to all available data introduced in \S\ref{sec:dataset}.
Our reconstruction is based on individual characters, while existing results are based on phonological categories. 
A straightforward way to obtain category-centric result is averaging phonetic feature vectors of all characters belonging to the same phonological category. 
The nearest IPA phonemes to the averaged vectors can be directly used for comparison to previous manual results by philologists.
\footnote{A comprehensive summary of the result can be found at \url{https://github.com/LuoXiaoxi-cxq/Reconstruction-of-Middle-Chinese-via-Mixed-Integer-Optimization}.}

\subsection{Numerical Evaluation} \label{sec:eval-AMI}
The phonetic vectors resulted from our model should form clusters that align with phonological categories in \gy~to some extent.
Based on this assumption, we develop a clustering based method to evaluate the overall quality of a reconstruction result. 
We cluster the phonetic vectors with KMeans with a predefined number of clusters equal to 37\footnote{There are 38 categories in the manual categorial reconstruction of \gy. 
Our dataset excludes characters with the \ch{俟} category.}.
We then report the \textit{adjusted mutual information}  \citep[AMI;][]{AMI-2010}, an information-theoretic measure, between the automatic clustering and predefined phonological categories. 
Given two clusterings $U$ and $V$, 
$$\mbox{AMI}(U,V)=\frac{\mbox{MI}(U, V) - \mathbb{E}[\mbox{MI}(U, V)]}{\mbox{avg}(\mbox{H}(U), \mbox{H}(V)) - \mathbb{E}[\mbox{MI}(U, V)]}$$
where $\mbox{H}$ is the Shannon entropy, $\mbox{MI}$ is the mutual information, and $\mathbb{E}$ is expectation.
Perfectly matched clusters yield an AMI of 1, while random cluster assignment yields 0.
The numbers of samples and clusters are not necessarily the same.

\begin{figure}[!t]
    \centering
    \includegraphics[width=\linewidth]{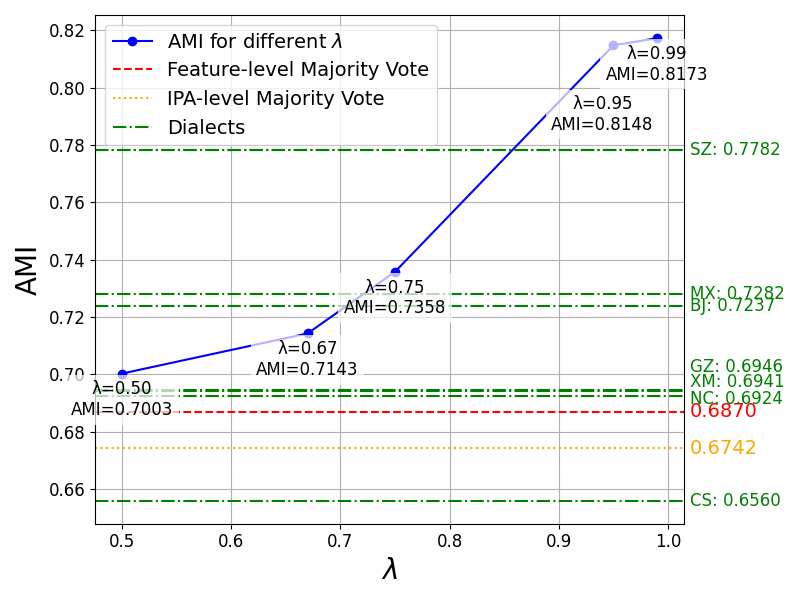}
    \caption{AMI values with different $\lambda$ compared to baseline results. 
    AMI values for seven individual dialects (BJ: Beijing, SZ: Suzhou, CS: Changsha, NC: Nanchang, MX: Meixian, GZ: Guangzhou, and XM: Xiamen) are presented, each representing one dialect group (see \S\ref{sec:dataset}).
    The single best dialect is Suzhou, with the highest AMI (0.7782). In feature-level majority vote, features are aggregated and voted individually. The KMeans algorithm is then applied to the voted feature vector, yielding an AMI of 0.6870.
    }
    \label{fig:AMI-our-baseline}
\end{figure}

Fig. \ref{fig:AMI-our-baseline} shows our results, along with two baselines (majority vote and single best dialect).
AMI is largely influenced by the value of $\lambda_{\text{fq}}$, indicating the effectiveness of \fq~in deriving categories.

Since we are dealing with 20 dialects but only a single set of \fq~spellings, setting $\lambda_{\text{fq}}$ to 0.95 is a natural choice. 
However, since the information obtained from the dialects and \fq~lacks a common scale, it is difficult to make direct comparisons.
As a result, we cannot theoretically determine the optimal weighting for each information source, and the choice of $\lambda$ is therefore largely empirical. 

When $\lambda_{\text{fq}}$ is set to 0.95, our model achieves an AMI of 0.8148 and outperforms the baselines, indicating a high degree of similarity between the phonetic reconstruction by us and the manual phonological reconstruction by philologists. 
When $\lambda_{\text{fq}}$ is set to 0.99, the AMI increases to 0.8173, although the change is minimal.

In contrast to Table \ref{tab:synthesis-adjust-parameters}, where setting $\lambda_{\text{fq}}$ to 0.95 results in a decrease in equal rate, increasing $\lambda_{\text{fq}}$ actually improves AMI when applied to real data. This phenomenon reflects the difference in synthetic and real data, emphasizing the importance of adjusting $\lambda_{\text{fq}}$ based on specific situations.

It is worth noting that when $\lambda_{\text{fq}}$ is set to 0 (i.e., without using \fq), AMI drops to 0.5892. This further reflects the critical role of \fq~information when dealing with real data.

In some auxiliary experiments that are not reported in this paper, we also used hierarchical clustering to further study the impact of the number of clusters. Results show that the difference in AMI between hierarchical clustering and KMeans is within 0.05, regardless of the specific setting.

\subsection{Comparison to Existing Results} \label{sec:comparison}
Our model successfully reconstructs most categories with consensus among philologists, such as b\=ang 幫 p\=ang 滂 m\textipa{\'\i}ng 明 du\=an 端 t\`ou 透 n\textipa{\'\i} 泥\footnote{All the Chinese characters used in \S\ref{sec:comparison} are categorical labels representing initial categories in traditional Chinese phonology. For example, characters f\=ang 方, f\v u 府,  b\'o 博, b\textipa{\v\i} 彼, and many other characters are assumed to have the same initial in MC, and philologists use b\=ang 幫 to represent their common initials.}.
Similar to the computational operationalisation of the historical comparative approach to Indo-European languages \citep{list-etal-2022-new}, our study confirms the usefulness of computation in linguistic inquiry.  
The differences between different reconstruction results may provide new evidence for philologists and linguists to consider and therefore are useful too.
Such differences are mainly attributed to two factors.

Different dialects changed in different directions. 
For example, it is generally believed that Wu dialects retained all the voiced stops, while most other dialects became devoiced \citep[pp.224--225]{huang2014handbook}.\footnote{For example, the character t\'ong 同 is believed to had a voiced initial d\textipa{\`\i}ng 定 in MC. In Wu dialect, its initial is [d], while in most other dialects, it is [t\textipa{\super{h}}].}
In phonology, devoicing refers to a sound change where a voiced consonant becomes voiceless due to the influence of its phonological environment. This process is common across many languages and is a part of Grimm's law. 
Our current model, however, cannot differentiate in what aspects a dialect changed most and in what aspects it stayed constantly.
It treats different dialects with equal weight on different phenomena.
Consequently, a character's reconstructed initial tends to be closer with the phonetic value that is more commonly observed across various dialect pronunciations. 
For example, our model fails to reconstruct the `voiced' feature for categories that are assumed to be voiced by philologists, e.g. b\textipa{\`\i}ng 並 d\textipa{\`\i}ng 定 c\'ong 從 xi\'e 邪. 
Our model also has difficulty distinguishing the zh\textipa{\=\i} 知 zhu\=ang 莊 zh\=ang 章 groups, which have similar pronunciations in most modern Chinese dialects.

Our current model only contains the most basic information---\fq~spellings and modern Chinese varieties. 
Other types of information, including rhyme tables, e.g. Y\`unj\`ing 韻鏡, and seno-xenic\footnote{Sino-Xenic vocabularies are large-scale and systematic borrowings of the Chinese lexicon into the Japanese, Korean and Vietnamese languages. See \url{https://en.wikipedia.org/wiki/Sino-Xenic_vocabularies} for details.} pronunciations are not integrated into our model at present. 
Rhyme tables provide additional information about the voiced/voiceless feature of initials, which is crucial for philologists' manual reconstruction. 
It is another reason why our model fails to reconstruct voiced initials.

Because of the limitation in information sources, our model cannot provide definitive answers to some debatable problems, such as 
    whether categories n\'i 泥 and n\'iang 孃 are the same initial. 
It is generally believed that there is no distinction between the two categories in most modern Chinese varieties (\citealp[p.228]{Tseng-ni-niang}, \citealp[p.126]{lr-1956}).
Though \citet[pp.125--126]{lr-1956} and \citet[pp.98--101]{shaorongfen} have opposite opinions about this problem, they both used Sanskrit-Chinese pronunciations as the main evidence. 
However, in our model, with materials restricted to dialects, the reconstruction of n\'i 泥 and n\'iang 孃 appears similar.

\subsection{Extension}
In principle, our method can be generalized to other languages. However, in practice, our model requires phoneme-level alignment between each protoform's reflexes. 
For Chinese, this alignment occurs naturally, as each Chinese character typically corresponds to a morpheme, and morphemes are largely represented by single syllables that follow specific patterns, as described in \S\ref{sec:syllable-structure}.

Sound change is a central focus in linguistic research, and our model can engage with it in two ways.
First, incorporating common patterns of sound change as constraints into our model is a possible future direction. 
Second, by analyzing $F_{l}(X)-F_{\text{MC}}(X)$ in Eq. \ref{eq:objective}, we may identify potential sound changes in terms of distinctive features, such as devoicing.

\section{Related Work} \label{related-work}

\subsection{Computational Reconstruction} \label{related-work-reconstruct}

 \citeauthor{bouchard-2007a} (\citeyear{bouchard-2007a, bouchard-2007b,bouchard-2009,bouchard-2013}) proposed a series of influential work about unsupervised proto-word reconstruction, which requires an existing phylogenetic tree to infer the ancient word forms based on probability estimates for all the possible phoneme-level edits on each branch of the tree. The edit model parameters and unknown ancestral forms are jointly learned with an EM algorithm. 

Following this series of work, \citet{he-etal-2023-neural} also used Monte-Carlo EM algorithm but neural networks to parameterize the edit models, in order to express more complex phonological and the nonadjacent changes, achieving a notable reduction in edit distance from the target word forms. However, his highly parameterized edit models were designed for large cognate datasets with few languages, and may not be possible to train them on datasets with more languages but fewer datapoints per language. 

In supervised protolanguage reconstruction, the models are easier to evaluate. \citet{meloni-2021} trained a GRU-based encoder-decoder architecture on cognates from five Romance languages to predict their Latin ancestors, and achieved low error from the ground truth. \citet{kim-etal-2023-transformed} updated \citeauthor{meloni-2021}'s model with the Transformer and achieved better performance.

\citet{list-etal-2022-new} proposed a new framework for supervised reconstruction that combines automated sequence comparison with phonetic alignment analysis, which deals with the losing reflexes problem, and sound correspondence pattern detection, which models phonetic environments of sound change. 

\citet{lu-2024-improved-neural} proposed a multi-model reconstruction system that improves its reconstructions via predicting the reflexes given a protoform. Their system consists of a beam search-enabled sequence-to-sequence reconstruction model and a sequence-to-sequence reflex prediction model that serves as a reranker, surpassing state-of-the-art protoform reconstruction methods on three of four Chinese and Romance datasets.

\subsection{Middle Chinese Phonology} \label{related-work-MCP}

Phonetic reconstruction of phonological categories was pioneered by \citet{gbh}. 
Following the methodology of \citet{gbh}, subsequent scholars, including \citet{lfk-1971}, \citet{wangli-1957}, \citet{pulleyblanks}, \citet{Baxter1992}, made modifications to the methodology and proposed their reconstructions of MC.

In recent decades, some scholars have questioned the assumptions, methodology and conclusions of Karlgren's approach.
A critical view is exemplified by \citet{norman-1995}. 
Norman advocated a data-centered approach to Chinese historical phonology, predicated on the collection, analysis, and comparison of spoken-language data. 
His controversial reconstruction of Proto-Min (\citealp{Norman-1973}, \citeyear{norman-1974}) is an example.
\section{Conclusion} \label{conclusion}

We propose a novel, MIP-based method for phonetic reconstruction for Middle Chinese,
and validate its effectiveness on a wide range of synthesis and real data. 
Similar to the automation of the historical comparative approach to Indo-European languages, our study confirms the usefulness of computation in linguistic inquiry.  
The optimisation-based architecture is flexible---different information can be integrated as either an element in the objective function,  constraints, or both.
It is also applicable to the reconstruction problem of other languages.
We leave both for future work.

\section*{Acknowledgment}
We would like to express our sincere gratitude to the reviewers for their valuable comments, which greatly broadened our perspective and significantly improved the quality of our work. We would also like to thank Kechun Li for her suggestions.

\bibliographystyle{plainnat}
\bibliography{ref}

\section*{Appendix} \label{appendix}
Here, we prove that the distance function (\ref{distance-func}) defined in \S\ref{sec:dis-f} is mathematically sound.

It is easy to see:
\begin{enumerate}
  \item $g_{j,k}(F_1, F_2)\geqslant 0$. 
  \item $g_{j,k}(F_1, F_2) = g_{j,k}(F_2, F_1)$. 
\end{enumerate}

Now we consider the triangle inequality.
\begin{table*}[t!]
  \centering
    \begin{tabular}{rccll}
    \toprule
    \textbf{feature} & \textbf{[m]} & \textbf{[f]} & \textbf{dist.} & \multicolumn{1}{c}{\textbf{Note}} \\
    \midrule
    continuant & -1    & 1     & \textbf{2} & \\
    delayed release & 0     & 1     & \textcolor{blue}{\textit{2}}$^{\clubsuit}$ &  $^{\clubsuit}$ \small $ j=\text{delayed release}, \tau(j)=\text{sonority}$. $c=\min\{1, f(F_1^{\tau{(j)}}, F_2^{\tau{(j)}})\}=1, $ \\
    sonority & 2     & 1     & \textbf{1} & \small $s_j=2,\  g_{j,\tau{j}}(F_1, F_2)=c \cdot 2 + (1-c) f(F_1^j, F_2^j)=2$  \\
    voice & 1     & -1    & \textbf{2} &  \\
    spread glottis & -1    & -1    & \textbf{0} &  \\
    labial & 1     & 1     & \textbf{0} & $^{\diamondsuit}$\small $ j=\text{labiodental}, \tau(j)=\text{labial}$. $c=\min\{1, f(F_1^{\tau{(j)}}, F_2^{\tau{(j)}})\}=0, $ \\
    labiodental & -1    & 1     & \textcolor{blue}{\textit{2}}$^{\diamondsuit}$ & \small $ s_j=2, \ g_{j,\tau{j}}(F_1, F_2)=c \cdot 2 + (1-c) f(F_1^j, F_2^j)=f(F_1^j, F_2^j)=2$ \\
    coronal & -1    & -1    & \textbf{0} &  \\
    anterior & 0     & 0     & \textcolor{blue}{\textit{0}} &  \\
    distributed & 0     & 0     & \textcolor{blue}{\textit{0}} &  \\
    lateral & -1    & -1    & \textbf{0} &  \\
    dorsal & -1    & -1    & \textbf{0} &  \\
    high  & 0     & 0     & \textcolor{blue}{\textit{0}} &  \\
    front & 0     & 0     & \textcolor{blue}{\textit{0}} & \\
    \hline
    & & & \multicolumn{2}{l}{\textbf{total distance: 9}} \\
    \bottomrule
    \end{tabular}%
  \caption{An example of calculating the distance between [z] and [f] with our distance function. The distance between I-features (in bold) is calculated using general distance function $f(x_1, x_2)$, while the distance between D-features (in italic blue) is calculated using $g_{j,k}(F_1, F_2)$. The `total distance' is the sum of the distances across all dimensions. Details of the calculation are provided in the `Note' column.}
  \label{tab:exp-calcuate-dist}%
\end{table*}%

\begin{theorem} \label{proposition}
    $\forall F_1,F_2,F_3\in\Omega$, $\forall$ indices of paired D-feature and I-feature $j$ and $k$,
    \begin{equation}\label{eq:triangle}
    \setlen
        g_{j,k}(F_1,F_2)+g_{j,k}(F_1,F_3)\geqslant g_{j,k}(F_2,F_3).
    \end{equation}
\end{theorem}

\begin{proof}
    Let $c_{st}=\min\{f(X_s^k,X_t^k),1\}$. We have
    \begin{equation}
    \begin{aligned}
        \setlen
        \Delta=& g_{j,k}(F_1,F_2)+g_{j,k}(F_1,F_3)- g_{j,k}(F_2,F_3)\\
        =&(c_{12}+c_{13}-c_{23})\cdot s_j+(1-c_{12}) f(F_1^j,F_2^j)\\
        & +(1-c_{13}) f(F_1^j,F_3^j)-(1-c_{23}) f(F_1^j,F_2^j)\\
        \geqslant &(c_{12}+c_{13}-c_{23})\cdot s_j+(1-c_{12}) f(F_1^j,F_2^j)\\
        & +(1-c_{13})f(F_1^j,F_3^j) \\
        & -(1-c_{23})[f(F_1^j,F_2^j)+f(F_1^j,F_3^j)]\\
        = &(c_{12}+c_{13}-c_{23}) \cdot s_j-(c_{12}-c_{23}) f(F_1^j,F_2^j)\\
        & -(c_{13}-c_{23}) f(F_1^j,F_3^j)  \label{eq:need-simplify}
    \end{aligned}
    \end{equation}
    There are three cases: If $c_{23}<c_{12}$ and $c_{23}<c_{13}$, then
    \begin{equation*}
    \begin{aligned}
    \setlen
    \Delta & \geqslant (c_{12}-c_{23})[s_j - f(F_1^j,F_2^j)]\\
    & +(c_{13} - c_{23})[s_j-f(F_1^j,F_3^j)] \geqslant 0.
    \end{aligned}
    \end{equation*}
    If $c_{23}>c_{12}$ and $c_{23}>c_{13}$, then
    \begin{equation*}
    \begin{aligned}
    \setlen
     \Delta
        \geqslant (c_{12}+c_{13}-c_{23})\cdot s_j \geqslant0.
    \end{aligned}
    \end{equation*}
    If $c_{23}$ lies between $c_{12}$ and $c_{13}$, without loss of generality, assume $c_{12}\leqslant c_{23}\leqslant c_{13}$. Then,
    \begin{equation*}
    \begin{aligned}
    \setlen
        \Delta
        & \geqslant (c_{13}-c_{23})\cdot s_j-(c_{13}-c_{23}) f(F_1^j,F_3^j)\\
        & \geqslant( c_{13}-c_{23})[s_j-f(F_1^j,F_3^j)]\geqslant0.
    \end{aligned}
    \end{equation*}
\end{proof}

Table \ref{tab:exp-calcuate-dist} provides an example of calculating the distance between [m] and [f] using our distance function $d(F_1, F_2)$ defined in Eq. \ref{eq:dist-func}. The phonetic feature vectors of [m] and [f] are denoted $F_1$ and $F_2$, respectively. The general distance $f$ is set as $f(x_1, x_2)=|x_1-x_2|$.

\end{CJK}
\end{document}